\documentclass[10pt,letterpaper]{article}
\usepackage[T1]{fontenc}
\usepackage{times}
\usepackage[textwidth=5.75in,textheight=9in,centering]{geometry}
\usepackage{microtype}
\usepackage{titlesec}
\usepackage{titling}
\usepackage{enumitem}
\usepackage[authoryear,round]{natbib}
\setcitestyle{citesep={;},aysep={,},yysep={;}}

\titleformat{\section}{\large\scshape\raggedright}{\thesection}{1em}{}
\titleformat{\subsection}{\normalsize\scshape\raggedright}{\thesubsection}{1em}{}
\titleformat{\subsubsection}{\normalsize\scshape\raggedright}{\thesubsubsection}{1em}{}
\titlespacing*{\section}{0pt}{2ex plus .5ex minus .2ex}{1.5ex plus .3ex minus .2ex}
\titlespacing*{\subsection}{0pt}{1.8ex plus .5ex minus .2ex}{.8ex plus .2ex}
\titlespacing*{\subsubsection}{0pt}{1.5ex plus .5ex minus .2ex}{.5ex plus .2ex}
\setlist{topsep=4pt,itemsep=2pt,parsep=0pt,partopsep=0pt}
\pretitle{\noindent\LARGE\scshape\raggedright}
\posttitle{\par\vspace{1em}}
\preauthor{\noindent\normalsize\begin{tabular}[t]{@{}l@{}}}
\postauthor{\end{tabular}\par}
\renewcommand{\and}{\end{tabular}\hspace{10em}\begin{tabular}[t]{@{}l@{}}}
\predate{}
\postdate{}
\renewenvironment{abstract}
  {\begin{center}\large\scshape\abstractname\end{center}\begin{quote}\normalsize}
  {\end{quote}\vspace{1ex}}
\usepackage{amsmath,amsfonts,bm}

\def\eqref#1{equation~\ref{#1}}

\def\1{\bm{1}}

\DeclareMathAlphabet{\mathsfit}{\encodingdefault}{\sfdefault}{m}{sl}
\SetMathAlphabet{\mathsfit}{bold}{\encodingdefault}{\sfdefault}{bx}{n}

\newcommand{\E}{\mathbb{E}}

\newcommand{\R}{\mathbb{R}}

\newcommand{\KL}{D_{\mathrm{KL}}}
\newcommand{\Var}{\mathrm{Var}}

\DeclareMathOperator{\vecop}{vec}
\DeclareMathOperator{\tr}{tr}

\usepackage{url}
\usepackage{amsmath}
\usepackage{amsthm}
\usepackage{algorithm}
\usepackage{algpseudocode}
\usepackage{graphicx}
\usepackage{aliascnt}
\usepackage[hidelinks]{hyperref}
\hypersetup{
  pdftitle={WaterKron and FlipFlop Hessian: Information-Theoretically Grounded Quantization with Kronecker-Factored Hessians},
  pdfauthor={Johann Birnick and Rayan Saab}
}
\usepackage[capitalize,noabbrev,nameinlink]{cleveref}
\makeatletter
\def\theHALG@line{\thealgorithm.\arabic{ALG@line}}
\makeatother
\usepackage{mathtools}
\usepackage[para]{footmisc}

\theoremstyle{plain}
\newtheorem{theorem}{Theorem}[section]
\newaliascnt{lemma}{theorem}

\aliascntresetthe{lemma}
\crefname{lemma}{Lemma}{Lemmas}
\newaliascnt{corollary}{theorem}
\newtheorem{corollary}[corollary]{Corollary}
\aliascntresetthe{corollary}
\crefname{corollary}{Corollary}{Corollaries}
\theoremstyle{remark}
\newtheorem*{remark}{Remark}

\newcommand{\Z}{\mathbb{Z}}
\newcommand{\cN}{\mathcal{N}}

\title{WaterKron and FlipFlop Hessian:\\ Information-Theoretically Grounded Quantization with Kronecker-Factored Hessians}

\author{
\textbf{Johann Birnick} \\
Department of Mathematics \\
University of California San Diego \\
La Jolla, CA 92093, USA \\
\texttt{jbirnick@ucsd.edu}
\and
\textbf{Rayan Saab} \\
Department of Mathematics and HDSI \\
University of California San Diego \\
La Jolla, CA 92093, USA \\
\texttt{rsaab@ucsd.edu}
}

\date{}

\begin{document}

\maketitle

\begin{abstract}
How should a Kronecker-factored Hessian approximation be chosen for post-training quantization?
We address this question through \emph{WaterKron}, which combines two-sided GPTQ with row- and column-dependent waterfilling scales and entropy coding.
We derive its high-rate distortion with respect to the full Hessian using an explicit Kronecker-Hessian mismatch factor $\Phi$.
This factor quantifies the asymptotic distortion penalty due to the Kronecker Hessian approximation and provides a criterion for selecting the factors optimally.
Minimizing $\Phi$ leads to a Gaussian covariance-fitting problem with classical ``flip-flop'' updates.
We thus give a rate-distortion justification for using the resulting \emph{FlipFlop Hessian} in quantization.
We evaluate it empirically, finding that the FlipFlop Hessian consistently improves KL divergence and perplexity over input-only, marginal, and Frobenius-based Hessian choices.
\end{abstract}

\section{Introduction}
\label{sec:introduction}

\emph{Post-training quantization} of neural networks is the process of taking a fully trained neural network and converting its high-precision weights into a low-precision format, with the goal of retaining network accuracy while reducing memory requirements.
Usually, each linear module $W$ is quantized separately.
Hence the goal is to find a low-precision approximation $V$ of $W$.

There exist different classes of algorithms for finding such a $V$, which can also be combined.
Some methods focus on making the entries of $W$ incoherent and thus easier to quantize, for example by utilizing random or learned rotations \citep{chee2023quip,ashkboos2024quarot,liu2025spinquant}.
Other methods change the base quantizer that is applied to a single entry or a small group of entries, for example by using codebooks or lattices instead of simple rounding \citep{tseng2024quip,savkin2025nestquant}.
A third class of methods does ``adaptive rounding'' by conditioning the quantization of later entries within $W$ such that they correct the errors introduced by quantizing earlier entries.

Our work falls within the realm of this third class of methods, more concretely within its largest subclass, which might be labeled ``error propagation under a quadratic loss''.
A prominent example within this line of research is the GPTQ algorithm by \citet{frantar2023optq}.
In order to measure the distortion when $W$ is replaced by $V$, the algorithm uses a quadratic loss.
However, instead of just looking at $\lVert V - W \rVert_F$, for which simple ``round to nearest'' is the optimal quantization algorithm, it takes into account the input $x \in \R^n$ of the linear module $W$, and measures the distortion as $\E \lVert V x - W x \rVert_2^2$.
In practice, the expectation is taken over a calibration dataset.
This algorithm has been analyzed and extended from different directions \citep[e.g.,][]{chee2023quip,tseng2024quip,tseng2024qtip,zhang2026qronos,tseng2026modelpreserving,birnick2026lattice,chen2026geometry,zhang2025provable,ordentlich2026highratematmul2,lifar2026watersic}, but central to this paper are:
\begin{enumerate}
\item The extension to a \textbf{two-sided geometry} \citep{kim2025boa,tseng2026modelpreserving,birnick2026bakron}.
  In standard GPTQ, the rows of $W \in \R^{m \times n}$ are quantized independently.
  Formally, when we vectorize the matrices $W,V$ into $w = \vecop(W) \in \R^{m n}, v = \vecop(V)$, then the Hessian of the loss $\E \lVert (V - W) x \rVert_2^2 /2$ with respect to $v$ is $\E[x x^T] \otimes I_m$.
  While the $\E[x x^T]$ factor corresponds to the input geometry that GPTQ exploits, the identity factor $I_m$ means that all output features are treated independently and identically.
  There is a natural extension of GPTQ that can deal with arbitrary Kronecker-factored Hessians $A \otimes B$ and reduces to standard GPTQ if $B = I_m$.
  This variant of the algorithm effectively minimizes the loss $\lVert (B^{1/2})^T (V - W) A^{1/2} \rVert_F^2$; in particular, the factor $B$ in the Hessian also informs the algorithm of the output geometry of the linear module $W$.
  For the exact choice of the Hessian factors $A$ and $B$, multiple options have been explored in the literature (see \cref{sec:flipflophessian} and \cref{sec:experiments}).

\item The \textbf{information-theoretic} view and waterfilling formula \citep{lifar2026watersic,ordentlich2026highratematmul2}.
  Standard GPTQ allocates the same number of bits for every entry of $V$.
  An alternative is to use \emph{entropy coding} \citep{chen2026geometry,lifar2026watersic} for the quantized matrix.
  In particular, this allows us to ask about precision/bit/rate allocation:
  We have the freedom to store some entries in ``higher precision'' (usually costing more bits), and other entries in ``lower precision'' (using fewer bits).
  Concretely, ``precision'' corresponds to the scaling factors of the quantization grid.
  For example, Huffman-GPTQ \citep{chen2026geometry}, which does use Huffman coding and thus encodes some entries with more bits and others with fewer, still distributes the precision/rate uniformly across all entries.
  If the variance $\Var(x_j)$ of the $j$\textsuperscript{th} input channel is known, a more effective way to distribute precision would be to assign to the $j$\textsuperscript{th} column of $W$ a scale proportional to $1/\sqrt{\Var(x_j)}$.
  However, \citet{ordentlich2026highratematmul2,lifar2026watersic} notice that even that is suboptimal:
  In GPTQ, we can exploit linear correlation between different input channels, and what matters is only the residual variance that remains after accounting for this correlation.
  So the scale for the $j$\textsuperscript{th} column should be chosen proportional to $1/\sqrt{\min_{e_{j+1},\dots,e_n}\Var(x_j - e_{j+1} x_{j+1} - \dots - e_n x_n)}$, and this is essentially optimal.
\end{enumerate}

In this paper, we apply the information-theoretic lens to the two-sided algorithm.
First, this yields \emph{WaterKron}, an extension of the existing waterfilling allocation, WaterSIC, to the Kronecker-factored geometry.
Then we characterize its distortion in the full Hessian geometry, identify the penalty caused by the Kronecker approximation, and use it to find the optimal choice of Kronecker approximation.
Concretely, we make the following contributions:

\begin{enumerate}
\item \textbf{WaterKron: two-sided precision allocation.}
We extend the high-rate waterfilling results to two-sided GPTQ, deriving optimal row and column scales for the analyzed
scalar-grid alphabet.
Under the Gaussian and entropy-coding assumptions, this retains the approximately $0.255$-bit asymptotic gap to the
information-theoretic limit for the Kronecker distortion (\cref{thm:twosidedgptq}).
We demonstrate empirically that this outperforms allocating the waterfilling precision using only the input-sided Hessian (WaterSIC).

\item \textbf{Quantifying the cost of Hessian mismatch.}
We characterize WaterKron's high-rate distortion in the \emph{full} Hessian
geometry under the same assumptions. We introduce a mismatch factor $\Phi$
that quantifies the multiplicative distortion penalty caused by the
Kronecker approximation, providing a criterion for choosing its factors
(\cref{thm:twosidedgptq2}).

\item \textbf{The FlipFlop Hessian.}
Minimizing the mismatch factor, we recover the flip-flop algorithm from Gaussian covariance fitting.
We propose the resulting \emph{FlipFlop Hessian approximation}, which is justified by its distortion-minimizing objective and, to the best of our knowledge, has not been explored in this line of quantization research (\cref{sec:flipflophessian}).
We demonstrate empirically that it outperforms previous choices of Hessian approximations from the quantization literature.
\end{enumerate}

Our results extend the high-rate waterfilling theory of
\citet{ordentlich2026highratematmul2,lifar2026watersic} to Kronecker-factored quadratic losses.
They recover (and improve), as a special case, the row- and column-dependent grid allocation of SoftWater
\citep{cavalcanti2026softwater},
which is designed for the softmax module with a diagonal approximation of the output curvature $B$.

The rest of the paper is structured as follows.
In \cref{sec:background} we recall the setting and previous work: standard GPTQ, its two-sided extension, and the waterfilling precision allocation.
In \cref{sec:twosidedwaterfilling} we first derive the optimal precision allocation for the two-sided algorithm.
In \cref{sec:mismatchfactor} we take an information-theoretic viewpoint on the optimal choice of the Kronecker-factored Hessian approximation.
We introduce the Kronecker-Hessian mismatch factor $\Phi$ and provide a variant of the rate-distortion result that takes the full Hessian into account.
In \cref{sec:flipflophessian}, informed by the mismatch theory, we derive the ideal Kronecker approximation, proposing the FlipFlop Hessian.
Lastly, we provide empirical evidence for the success of our methods in \cref{sec:experiments}.

\section{Setting and Background}
\label{sec:background}

In this section, we first describe the alphabet that is used for the remainder of the paper.
Then we recall the standard GPTQ algorithm, its two-sided extension, and the waterfilling precision allocation.

\textit{Notation.}
For an $m$ by $n$ matrix $M$, $\vecop(M) \in \R^{mn}$ is its column-major vectorization.
Assuming $m=n$, we abbreviate $|M| := \det M$.
If $M$ is additionally positive definite, $\mathrm{Cholesky}(M)$ denotes the unique lower-triangular matrix $L$ with positive diagonal entries $L_{i,i} > 0$ that satisfies $L L^T = M$.
Similarly, $\mathrm{RevCholesky}(M)$ denotes the unique such $L$ that satisfies $L^T L =M$.
$M^{1/2}$ denotes an arbitrary matrix $S$ satisfying $S S^T = M$, the most important choice for this paper being $S = \mathrm{RevCholesky}(M)^T$.
All PSD matrices are assumed to be invertible; in practice, a small damping $\kappa I$ can be added to ensure this.
Lastly, $\{0,1\}^* = \bigcup_{n=0}^\infty \{0,1\}^n$ is the set of finite binary strings.

\subsection{Alphabet}
\label{sec:bg_alphabet}

Recall that the goal of weight-only post-training quantization is to approximate $W \in \R^{m \times n}$ with a matrix $V$ that has lower precision.
For this paper, $V$ has the following structure:
\begin{itemize}
\item At its core, $V$ has essentially \emph{integer entries}.
  We denote this integral matrix by $Z \in \mathbb{Z}^{m \times n}$.
  The integers are unbounded (not clipped), and later encoded with a good coding scheme.
  Thus the memory cost of $Z$ is determined by the entropy of the distribution of its entries.
\item Additionally, there are real-valued \emph{scaling factors} per row and per column.
  Concretely, the $j$\textsuperscript{th} column of $V$ is scaled by $\alpha_j \in \R_{>0}$ and the $i$\textsuperscript{th} row of $V$ is scaled by $\beta_i \in \R_{>0}$.
  These input/output scaling factor vectors $\alpha = (\alpha_1,\dots,\alpha_n) \in \R^n$ and $\beta = (\beta_1,\dots,\beta_m) \in \R^m$ are stored in high precision; their practical memory cost is very small.
\end{itemize}

So $V$ is an integer matrix $Z$ whose rows and columns have been rescaled by $\beta$ and $\alpha$.
Concretely, the scaling factors form a rank-1 matrix $\beta \alpha^T \in \R^{m \times n}$, and we have:
\begin{equation*}
V = \beta \alpha^T \odot Z
\qquad\qquad
V_{i,j} = \alpha_j \beta_i \cdot Z_{i,j}
\end{equation*}

In the algorithms, we consider $V_{i,j}$ as implicitly defined as soon as $Z_{i,j}, \alpha_j, \beta_i$ are set.

Crucially, the scaling factors can be used for distributing precision across rows and columns.
A small scaling factor corresponds to high precision.
For example, if $\alpha_1$ is very small, that means the scaled integer grid of column 1 is very fine-grained, so column 1 is stored in relatively high precision.

\subsection{GPTQ: Adaptive Rounding}
\label{sec:bg_gptq}

\begin{algorithm}[h]
  \caption{GPTQ}
  \label{alg:gptq}
  \begin{algorithmic}
    \State \textbf{Input:} $W \in \R^{m \times n}$, $A = \E[x x^T] \in \R^{n \times n}$, $\alpha \in \R^n$, $\beta \in \R^m$
    \State $L \gets \mathrm{RevCholesky}(A)$
    \State $\tilde{L} \gets L^{-1} = \mathrm{Cholesky}(A^{-1})$
    \State Normalize columns of $\tilde{L}$ by the diagonal entries; i.e., replace $\tilde{L}_{:,j} \gets \tilde{L}_{:,j} / \tilde{L}_{j,j}$ for $j = 1,\ldots,n$.
    \For{$j = 1$ \textbf{to} $n$}
      \State $Z_{:,j} \gets \mathrm{round}(W_{:,j} / (\alpha_j \beta))$
      \State $\Delta \gets V_{:,j} - W_{:,j}$
      \State $W \gets W + \Delta \cdot (\tilde{L}_{:,j})^T$
    \EndFor
    \State \textbf{Output:} $Z$
  \end{algorithmic}
\end{algorithm}

We recall the standard GPTQ algorithm by \citet{frantar2023optq}, described in \cref{alg:gptq}.
The idea of this algorithm is to exploit the correlation between the $n$ input channels of the linear module $W \in \mathbb{R}^{m \times n}$.
It does so by aiming to minimize the following quadratic distortion:
\begin{equation}
\label{eq:distortiongptq}
D_A = \frac{1}{m n} \E \lVert V x - W x \rVert_2^2 = \frac{1}{m n} \lVert (V - W) \cdot \E[x x^T]^{1/2} \rVert_F^2 = \frac{1}{m n} \lVert (V - W) L^T \rVert_F^2
\end{equation}
The expectation here, and in \cref{sec:flipflophessian}, is taken over an (empirical) calibration data distribution, $x \in \R^n$ is an input to the linear module $W$, and $L = \mathrm{RevCholesky}(\E[x x^T])$.
The (reverse) Cholesky decomposition $L$ encodes linear correlation between the input channels, and sits at the heart of the GPTQ algorithm.
The algorithm quantizes the input channels sequentially, and modifies $W$ along the way in order to partially correct the quantization error introduced by quantizing earlier input channels, which is only possible due to the correlation between those channels.
See \cref{alg:gptq}.

\subsection{Two-Sided Extension and Kronecker-Factored Hessians}
\label{sec:bg_twosided}

Later \citep{kim2025boa,tseng2026modelpreserving,birnick2026bakron} it was realized that GPTQ can be naturally extended to also take the output geometry into account.
Indeed, the Hessian of the minimization problem $\E \lVert V x - W x \rVert_2^2 / 2$ with respect to $v = \vecop(V)$ is $\E[x x^T] \otimes I_m$.
While the $\E[x x^T]$ factor represents the input correlations that GPTQ exploits, the $I_m$ factor means that output channels are quantized independently.
The two-sided version of the algorithm instead considers an arbitrary Kronecker-factored Hessian $A \otimes B$, which corresponds to the following two-sided distortion:
\begin{align*}
D_{A,B} &= \frac{1}{m n} \lVert (B^{1/2})^T (V - W) A^{1/2} \rVert_F^2
 &&= \frac{1}{m n} \lVert (\vecop(V)^T - \vecop(W)^T) (A^{1/2} \otimes B^{1/2}) \rVert_F^2  \\
  &= \frac{1}{m n} \lVert L^{(B)} (V - W) (L^{(A)})^T \rVert_F^2 
 &&= \frac{1}{m n} \lVert (\vecop(V)^T - \vecop(W)^T) (L^{(A)} \otimes L^{(B)})^T \rVert_F^2
\end{align*}
Here $L^{(A)} = \mathrm{RevCholesky}(A)$ and $L^{(B)} = \mathrm{RevCholesky}(B)$, and since matrix products distribute over Kronecker factors, $L^{(A)} \otimes L^{(B)} = \mathrm{RevCholesky}(A \otimes B)$.
The identities from the right column, when compared to \cref{eq:distortiongptq}, already reveal how the natural extension of GPTQ to a two-sided geometry works:
We can just call standard GPTQ (\cref{alg:gptq}) with $\vecop(W)^T \in \R^{1 \times m n}$ as the weight matrix and $A \otimes B \in \R^{mn \times mn}$ as the input correlation matrix.
This is described in \cref{alg:twosidedgptq} in ``unvectorized'' notation.
Note that if $B = I_m$, this is equivalent to \cref{alg:gptq}.
For fast implementations of \cref{alg:twosidedgptq}, see \citet{birnick2026bakron, chen2026gptq2d}.

\begin{algorithm}[H]
  \caption{Two-sided extension of GPTQ}
  \label{alg:twosidedgptq}
  \begin{algorithmic}
    \State \textbf{Input:} $W \in \R^{m \times n}$, $A \in \R^{n \times n}$, $B \in \R^{m \times m}$, $\alpha \in \R^n$, $\beta \in \R^m$
    \State $L^{(A)}, L^{(B)} \gets \mathrm{RevCholesky}(A), \mathrm{RevCholesky}(B)$
    \State $\tilde{L}^{(A)}, \tilde{L}^{(B)} \gets (L^{(A)})^{-1}, (L^{(B)})^{-1}$
    \State Normalize columns of $\tilde{L}^{(A)}, \tilde{L}^{(B)}$ by the diagonal entries.
    \For{$j = 1$ \textbf{to} $n$}
      \For{$i = 1$ \textbf{to} $m$}
        \State $Z_{i,j} \gets \mathrm{round}(W_{i,j} / (\alpha_j \beta_i))$
        \State $\Delta \gets V_{i,j} - W_{i,j}$
        \State $W \gets W + \tilde{L}^{(B)}_{:,i} \cdot \Delta \cdot (\tilde{L}^{(A)}_{:,j})^T$
      \EndFor
    \EndFor
    \State \textbf{Output:} $Z$
  \end{algorithmic}
\end{algorithm}

\subsection{Information-Theoretic View and Waterfilling Formula}
\label{sec:bg_waterfilling}

Note that GPTQ (\cref{alg:gptq}) and its two-sided extension (\cref{alg:twosidedgptq}) are at their core just about error propagation using the (reverse) Cholesky factors $L$, and are essentially independent of the alphabet that is used to store the entries.
In particular, they do not postulate specific scaling factors $\alpha, \beta$.
Instead, those scaling factors are chosen up front and handed to the algorithm.
For example, Huffman-GPTQ \citep{chen2026geometry} implicitly chooses $\alpha \propto (1,\dots,1)$ and $\beta \propto (1,\dots,1)$, i.e., a single scalar scaling factor for the whole matrix.

As noted in \cref{sec:bg_alphabet}, the scaling factors can be used to distribute precision/rate across rows and columns non-uniformly.
\citet{ordentlich2026highratematmul2} realized this and derived information-theoretically optimal scaling factors in the case of standard (only input-sided) GPTQ, i.e., \cref{alg:gptq}.
Recall that in this setting we have
\begin{equation*}
A = \E[x x^T] \qquad L = \mathrm{RevCholesky}(A) \qquad \tilde{L} = L^{-1} = \mathrm{Cholesky}(A^{-1})
\end{equation*}
where the expectation is over calibration data.
\citet{ordentlich2026highratematmul2,lifar2026watersic} then postulate the input scaling factor choice $\alpha_j \propto 1/L_{j,j} = \tilde{L}_{j,j}$.
Since they only consider the standard input-sided algorithm, they implicitly choose uniform output scaling factors $\beta \propto (1,\dots,1)$.

They prove that, if $W_{i,j} \sim \cN(0,\sigma_W^2)$ is modeled as iid Gaussian, then \cref{alg:gptq} with this choice of scaling factors is an (almost) information-theoretically optimal quantization scheme for the distortion in \cref{eq:distortiongptq}.
Concretely, they consider arbitrary encoder-decoder pairs: The encoder $f\colon \R^{m \times n} \times \R^{n \times n} \to \{0,1\}^*$ takes $(W,A)$ and produces a self-delimiting (prefix-free) binary string, and the decoder $g\colon \{0,1\}^* \to \R^{m \times n}$ takes the string and produces $V$, so that we have $V = g(f(W,A))$.
For example, one such encoder-decoder pair is given by \cref{alg:gptq} together with an entropy coding scheme to encode/decode $Z$.
(Here we gloss over the encoding of $\alpha,\beta$; see \cref{sec:alphabetacoding}.)
For such a pair, the \emph{distortion} $D_A$ is measured as the expected input-sided distortion from \cref{eq:distortiongptq} and the \emph{rate} $R$ as the expected string length per entry.
\begin{equation}
\label{eq:ratedistortiondef}
D_A = \frac{1}{m n} \E_W \Big[ \lVert (V - W) \cdot A^{1/2} \rVert_F^2 \Big]
\qquad
R = \frac{1}{m n} \E_W \Big[ \mathrm{length}(f(W,A)) \Big]
\end{equation}
Note that at this point $A$ is seen as a deterministic Gram matrix, and the expectation runs over the Gaussian distribution of $W$.
In this setting, they prove the following theorems (reformulated).

\begin{theorem}[IT lower bound]
\label{thm:itlimit}
Suppose that $W_{i,j} \sim \cN(0,\sigma_W^2)$ is iid Gaussian and $A \in \R^{n \times n}$ is positive definite.
Then any encoder-decoder pair $(f,g)$ satisfies:
\begin{equation*}
D_A \geq \sigma_W^2 |A|^{1/n} 2^{-2R}
\end{equation*}
\end{theorem}

\begin{theorem}
\label{thm:gptq}
Suppose that $W_{i,j} \sim \cN(0,\sigma_W^2)$ is iid Gaussian and $A \in \R^{n \times n}$ is positive definite.
Additionally fix an entropy coding scheme that encodes $Z$ with expected length $H(Z) + o(m n)$.
Then \cref{alg:gptq} with scale choice $\alpha_j = \gamma \cdot |A|^{1 / 2n} / L_{j,j}$ and $\beta_i = 1$ achieves
\begin{equation*}
D_A = \frac{\gamma^2}{12} |A|^{1/n} \cdot (1 + o(1))
\qquad
\qquad
R = \frac{1}{2} \log(2 \pi e \sigma_W^2) - \log \gamma + o(1)
\end{equation*}
as $\gamma \to 0$ and $mn \to \infty$.
This yields the rate-distortion tradeoff
\begin{equation*}
D_A = \frac{2 \pi e}{12} \cdot \sigma_W^2 |A|^{1/n} 2^{-2R} \cdot (1 + o(1))
\end{equation*}
which, in the limit, is within $\frac{1}{2} \log(2 \pi e / 12) \approx 0.255$ bits of the limiting rate from \cref{thm:itlimit}.
\end{theorem}

The scaling factor choice $\alpha_j \propto 1/L_{j,j}$ is known as the high-rate waterfilling rule.
\cref{alg:gptq} with these scaling factors, and potentially with other GPTQ improvements on top, is called the ``WaterSIC'' algorithm, see \citet{lifar2026watersic}.

\begin{figure}[h]
  \centering
  \includegraphics[width=0.9\textwidth]{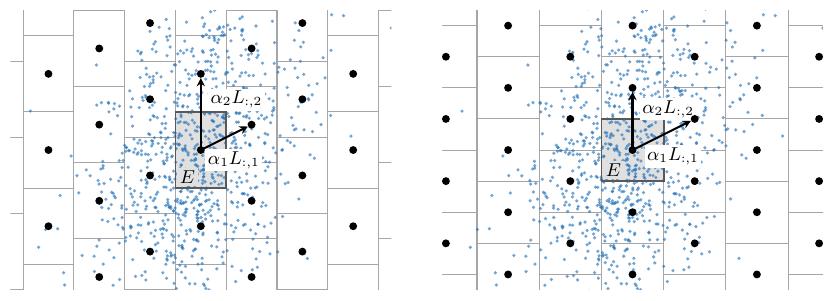}
  \caption{Babai rounding for $L=\bigl(\begin{smallmatrix}2&0\\1&3\end{smallmatrix}\bigr)$.
    Lattice points $L\operatorname{diag}(\alpha)\Z^2$ are represented by black dots.
    Target vectors $Lw$ with $w\sim\cN(0,I_2)$ are represented by blue dots.
    Each target is mapped to the lattice point at the center of its cell.
    The fundamental region $E$ that contains the quantization errors is highlighted.
    On the left, the scaling factors are $\alpha=(1,1)$, giving rectangular cells.
    On the right, they are chosen optimally as $\alpha=(3/\sqrt{6},\,2/\sqrt{6})$, giving square cells while preserving the volume.}
  \label{fig:babai}
\end{figure}

We will provide some intuition for why $\alpha_j \propto 1/L_{j,j}$ is the correct choice of scaling factors, as this will become important in \cref{sec:flipflophessian}.
For this, we take a lattice perspective \citep{birnick2026lattice,chen2026geometry}:
\cref{alg:gptq} is equivalent to running Babai's nearest plane algorithm \citep{babai1986lovasz} on the lattice $L \mathrm{diag}(\alpha) \Z^n$ with target vectors $L W_{i,:}^T$.
It tiles $\R^n$ with copies of a rectangular fundamental region, where the rectangle has side lengths $(\alpha_1 L_{1,1},\dots,\alpha_n L_{n,n})$.
Each target vector is quantized to the lattice vector at the center of the rectangle that it lies inside.
Denote the rectangle around the zero vector by $E = \prod_{j=1}^n [-\alpha_j L_{j,j}/2, \alpha_j L_{j,j}/2]$.
The quantization error vectors $L (V_{i,:} - W_{i,:})^T$ lie in this region.
In the high-rate regime (where the scaling factors approach zero), the error vectors are approximately uniformly distributed within $E$.
Thus the distortion is:
\begin{equation}
\label{eq:distortionuniform}
D_A = \frac{1}{n} \E_{W,i} \Big[ \lVert (V_{i,:} - W_{i,:}) L^T \rVert_F^2 \Big] \approx \frac{1}{n} \E_{e \sim \mathrm{Unif}(E)} \Big[ \lVert e \rVert_2^2 \Big] = \frac{1}{12} \frac{1}{n} \sum_{j=1}^n (\alpha_j L_{j,j})^2
\end{equation}
Moreover, fixing the rate $R$ corresponds to fixing the \emph{volume} of the rectangle $E$.
So minimizing distortion at a given rate corresponds to minimizing \cref{eq:distortionuniform} while leaving $\mathrm{Vol}(E) = \prod_{j=1}^n \alpha_j L_{j,j}$ fixed.
This is achieved by making the side lengths of $E$ all equal, so that it becomes a cube.
This yields $\alpha_j \propto 1/L_{j,j}$.
See \cref{fig:babai} for a visualization.

\section{Two-Sided Rate Allocation}
\label{sec:twosidedwaterfilling}

We aim to unify the two-sided \cref{alg:twosidedgptq}
with the information-theoretic view on \cref{alg:gptq}.
The first step is to choose the scaling factors $\alpha$ and $\beta$ for \cref{alg:twosidedgptq} appropriately.
Perhaps unsurprisingly, we show that the optimal choice is $\alpha_j \propto 1/L^{(A)}_{j,j}$ and $\beta_i \propto 1/L^{(B)}_{i,i}$.
This follows directly from viewing \cref{alg:twosidedgptq} as being equivalent to \cref{alg:gptq} on a transformed input.
We establish the corresponding information-theoretic and achievability theorems for the two-sided \cref{alg:twosidedgptq}.

Concretely, let $W_{i,j} \sim \cN(0,\sigma_W^2)$ iid Gaussian, let the encoder $f\colon \R^{m \times n} \times \R^{n \times n} \times \R^{m \times m} \to \{0,1\}^*$ take $(W,A,B)$, let the decoder $g\colon \{0,1\}^* \to \R^{m \times n}$ be as before, and define rate and distortion in the two-sided setting as follows:
\begin{equation}
\label{eq:twosidedratedistortiondef}
D_{A,B} = \frac{1}{m n} \E_W \Big[ \lVert (B^{1/2})^T (V - W) \cdot A^{1/2} \rVert_F^2 \Big]
\qquad
R = \frac{1}{m n} \E_W \Big[ \mathrm{length}(f(W,A,B)) \Big]
\end{equation}

Then we have the following results as the two-sided versions of \cref{thm:itlimit} and \cref{thm:gptq}.

\begin{corollary}[IT lower bound]
\label{thm:twosideditlimit}
Suppose that $W_{i,j} \sim \cN(0,\sigma_W^2)$ is iid Gaussian and $A,B$ are positive definite.
Then any encoder-decoder pair $(f,g)$ satisfies:
\begin{equation*}
D_{A,B} \geq \sigma_W^2 |A|^{1/n} |B|^{1/m} 2^{-2R}
\end{equation*}
\end{corollary}

\begin{proof}
Define $\hat{n} := m n$ and $\hat{m} := 1$, as well as $\hat{A} := A \otimes B \in \R^{\hat{n} \times \hat{n}}$ and $\hat{W} := \vecop(W)^T \in \R^{\hat{m} \times \hat{n}}$.
Then note that iid $\hat{W}_{i,j} \sim \cN(0,\sigma_W^2)$ and invoke \cref{thm:itlimit} with $\hat{W}$ and $\hat{A}$.
Lastly, note that the rate and distortion definitions \labelcref{eq:ratedistortiondef} for $\hat{A},\hat{W}$ transform exactly to the definitions \labelcref{eq:twosidedratedistortiondef} for $A,B,W$, e.g., $D_{A,B} = D_{\hat{A}}$, and that $|\hat{A}|^{1/\hat{n}} = |A \otimes B|^{1/(mn)} = (|A|^m |B|^n)^{1/(mn)} = |A|^{1/n} |B|^{1/m}$.
So the conclusion of our invocation of \cref{thm:itlimit} is exactly the conclusion of this corollary.
\end{proof}

\begin{corollary}
\label{thm:twosidedgptq}
Suppose that $W_{i,j} \sim \cN(0,\sigma_W^2)$ is iid Gaussian and $A,B$ are positive definite.
Additionally fix an entropy coding scheme that encodes $Z$ with expected length $H(Z) + o(m n)$.
Then \cref{alg:twosidedgptq} with scale choice $\alpha_j = \sqrt{\gamma} \cdot |A|^{1 / 2n} / L^{(A)}_{j,j}$ and $\beta_i = \sqrt{\gamma} \cdot |B|^{1 / 2m} / L^{(B)}_{i,i}$ achieves
\begin{equation*}
D_{A,B} = \frac{\gamma^2}{12} |A|^{1/n} |B|^{1/m} \cdot (1 + o(1))
\qquad
\qquad
R = \frac{1}{2} \log(2 \pi e \sigma_W^2) - \log \gamma + o(1)
\end{equation*}
as $\gamma \to 0$ and $mn \to \infty$.
This yields the rate-distortion tradeoff
\begin{equation*}
D_{A,B} = \frac{2 \pi e}{12} \cdot \sigma_W^2 |A|^{1/n} |B|^{1/m} 2^{-2R} \cdot (1 + o(1))
\end{equation*}
which, in the limit, is within $\frac{1}{2} \log(2 \pi e / 12) \approx 0.255$ bits of the limiting rate from \cref{thm:twosideditlimit}.
\end{corollary}

\begin{proof}
We again define $\hat{n} := m n$ and $\hat{m} := 1$, then $\hat{A} := A \otimes B \in \R^{\hat{n} \times \hat{n}}$ and $\hat{W} := \vecop(W)^T \in \R^{\hat{m} \times \hat{n}}$, and also $\hat{\alpha} = \alpha \otimes \beta \in \R^{\hat{n}}$ and $\hat{\beta} = 1 \in \R^{\hat{m}}$.
Define $\hat{L}, L^{(A)}, L^{(B)}$ as the reverse Cholesky decompositions of $\hat{A},A,B$, and note that $\hat{L} = L^{(A)} \otimes L^{(B)}$.
Here, again, we have iid $\hat{W}_{i,j} \sim \cN(0,\sigma_W^2)$ and we invoke \cref{thm:gptq} with $\hat{W}$ and $\hat{A}$.

Now note that calling \cref{alg:gptq} with $\hat{W},\hat{A},\hat{\alpha},\hat{\beta}$, upon unvectorizing the resulting $\hat{V}$ (or $\hat{Z}$) back into an $m$ by $n$ matrix, is equivalent to calling \cref{alg:twosidedgptq} with $W,A,B,\alpha,\beta$.
And if, for any pair $(i,j)$, we define $\hat{j} := (j-1)m + i$, then we have
\begin{equation*}
\hat{\alpha}_{\hat{j}} = (\alpha \otimes \beta)_{\hat{j}} = \alpha_j \beta_i = \sqrt{\gamma} \cdot |A|^{1/2n} / L^{(A)}_{j,j} \cdot \sqrt{\gamma} \cdot |B|^{1/2m} / L^{(B)}_{i,i} = \gamma \cdot |\hat{A}|^{1/2\hat{n}} / \hat{L}_{\hat{j}, \hat{j}}
\end{equation*}
which shows that the scaling factors of this corollary correspond to those prescribed by our invocation of \cref{thm:gptq}.
Again, the rate and distortion definitions \labelcref{eq:ratedistortiondef} and \labelcref{eq:twosidedratedistortiondef} match, $D_{A,B} = D_{\hat{A}}$, so the conclusion of our invocation of \cref{thm:gptq} is exactly the conclusion of this corollary.
\end{proof}

\begin{remark}
The only subtlety is the entropy coding scheme for $Z$.
For the algorithm from \cref{thm:gptq}, \citet{lifar2026watersic} suggest Huffman coding for every column, which is possible since all entries in one column $j$ have the same scaling factor $\alpha_j$, due to the row scaling factors $\beta = (1,\dots,1)$ being trivial.
For the algorithm in this corollary, this kind of Huffman coding is no longer efficient enough since both $\alpha$ and $\beta$ are nontrivial, and thus every entry $(i,j)$ has a unique scaling factor $\alpha_j \beta_i$.
However, other entropy coding schemes can deal with this, such as variants of arithmetic coding.
\end{remark}

\section{Kronecker-Hessian Mismatch Factor}
\label{sec:mismatchfactor}
So far, we have analyzed the two-sided algorithm with the distortion measure $D_{A,B}$, which implicitly assumes a Kronecker-factored Hessian structure $A \otimes B$.
However, in reality, the Kronecker product $A \otimes B$ is only an \emph{approximation} to some ``full Hessian'' $H \in \R^{mn \times mn}$ that truly captures the geometry of the problem, but is too big to directly deal with algorithmically.
In practice, $H$ is usually the Fisher information matrix of $W$, see \cref{sec:flipflophessian}.
The true distortion must be measured with respect to $H$, which leads to the following distortion measure:
\begin{equation}
\label{eq:fulldistortiondef}
D_H = \frac{1}{m n} \E_W \Big[ \lVert (\vecop(V)^T - \vecop(W)^T) \cdot H^{1/2} \rVert_F^2 \Big]
\end{equation}

We now show that when approximating $H$ by $A \otimes B$ and running \cref{alg:twosidedgptq} with $A,B$, the full distortion $D_H$ asymptotically equals the two-sided distortion $D_{A,B}$ from \cref{thm:twosidedgptq} times $\tr(H (A \otimes B)^{-1})/mn$.
Note that it is not generally true that $D_H = D_{A,B} \cdot \tr(H (A \otimes B)^{-1})/mn$.
In general there is a correction term, which for \cref{alg:twosidedgptq} in the setting of our theorem becomes negligible as the quantization errors in $E$ become asymptotically uniformly distributed.

\begin{theorem}
\label{thm:twosidedgptq2}
Suppose that $W_{i,j} \sim \cN(0,\sigma_W^2)$ is iid Gaussian and $H,A,B$ are positive definite.
Additionally fix an entropy coding scheme that encodes $Z$ with expected length $H(Z) + o(m n)$.
Then \cref{alg:twosidedgptq} with scale choice $\alpha_j = \sqrt{\gamma} \cdot |A|^{1 / 2n} / L^{(A)}_{j,j}$ and $\beta_i = \sqrt{\gamma} \cdot |B|^{1 / 2m} / L^{(B)}_{i,i}$ achieves
\begin{equation*}
D_H = \frac{\gamma^2}{12} |A|^{1/n} |B|^{1/m} \frac{\tr(H (A \otimes B)^{-1})}{mn} \cdot (1 + o(1))
\qquad
R = \frac{1}{2} \log(2 \pi e \sigma_W^2) - \log \gamma + o(1)
\end{equation*}
as $\gamma \to 0$ and $mn \to \infty$.
This yields the rate-distortion tradeoff:
\begin{equation*}
D_H = \frac{2 \pi e}{12} \cdot \sigma_W^2 |A|^{1/n} |B|^{1/m} \frac{\tr(H (A \otimes B)^{-1})}{mn} 2^{-2R} \cdot (1 + o(1))
\end{equation*}
\end{theorem}

\begin{proof}
The setting is as in \cref{thm:twosidedgptq}, which already provides the rate $R = \frac{1}{2} \log(2 \pi e \sigma_W^2) - \log \gamma + o(1)$.
We only have to compute the distortion $D_H$ with respect to the full Hessian.
Recall the intuition from \cref{sec:bg_waterfilling}.
For the two-sided algorithm, the ``error region'' in the lattice space is
\begin{equation*}
E := \prod_{j=1}^n \prod_{i=1}^m [-\alpha_j \beta_i L^{(A)}_{j,j} L^{(B)}_{i,i} / 2, \alpha_j \beta_i L^{(A)}_{j,j} L^{(B)}_{i,i} / 2]
\end{equation*}
The scaling factors are chosen exactly so that this becomes a cube.
Concretely, we have:
\begin{equation*}
E = |A|^{1/2n} |B|^{1/2m} \cdot \prod_{j=1}^n \prod_{i=1}^m [-\gamma / 2, \gamma / 2]
\end{equation*}
The transformed quantization error $\vecop(L^{(B)} (V - W) (L^{(A)})^T)$ lies in $E$.
For simplicity, we work only in the high-rate limit $\gamma \to 0$ and assume that this transformed quantization error vector is uniformly distributed in $E$.
As $\gamma \to 0$ this holds asymptotically because $W$ is normally distributed; intuitively, in \cref{fig:babai}, the rectangles/cubes become so small that the blue dots become approximately uniformly distributed within them.
For a concrete proof of this, see \citet[Appendix B.2]{lifar2026watersic}.
Write $L := L^{(A)} \otimes L^{(B)}$, $d := \vecop(V - W)$, and note that $e := L d$ is the transformed quantization error $\vecop(L^{(B)} (V - W) (L^{(A)})^T) \in E$.
So $d = L^{-1} e$, where $e$ is asymptotically uniformly distributed, giving:
\begin{align*}
D_H &= \frac{1}{mn} \E_W [d^T H d] = \frac{1}{mn} \E_W [\tr(H d d^T)] = \frac{1}{mn} \E_W[\tr(H L^{-1} e e^T L^{-T})] \\
&\overset{\mathclap{\gamma \to 0}}{\asymp} \frac{1}{mn} \tr(H L^{-1} \E_{e \sim \mathrm{Unif}(E)}[e e^T] L^{-T}) = \frac{1}{mn} \frac{|A|^{1/n} |B|^{1/m} \gamma^2}{12} \tr(H L^{-1} I L^{-T}) \\
&= \frac{\gamma^2}{12} |A|^{1/n} |B|^{1/m} \frac{\tr(H (A \otimes B)^{-1})}{mn} \qedhere
\end{align*}
\end{proof}

\cref{thm:twosidedgptq2} now gives us a clear criterion for choosing the Hessian approximation $A \otimes B$ in order to minimize distortion:
We have to minimize the term in the distortion that depends on $A,B$, which is $|A|^{1/n} |B|^{1/m} \tr (H (A \otimes B)^{-1})/mn$.
It is cleanest to additionally normalize this by $|H|^{1/mn}$.
We call the resulting metric the \emph{Kronecker-Hessian mismatch factor} and abbreviate it as $\Phi$:
\begin{equation*}
\Phi(H,A \otimes B) := \frac{\tr(H (A \otimes B)^{-1}) / mn}{|H|^{1/mn} \cdot |A|^{-1/n} |B|^{-1/m}} = \frac{\tr(H (A \otimes B)^{-1}) / mn}{|H|^{1/mn} \cdot |A \otimes B|^{-1/mn}}
\end{equation*}
The mismatch factor has the following additional interpretations.
First, define the \emph{generalized eigenvalues} $\lambda_1, \dots, \lambda_{mn}$ of $(H, A \otimes B)$ as the eigenvalues of $((A \otimes B)^{1/2})^{-1} H ((A \otimes B)^{1/2})^{-T}$.
Then $\Phi$ is the ratio of the arithmetic mean to the geometric mean of those generalized eigenvalues:
\begin{equation*}
\Phi(H, A \otimes B) = \frac{\frac{1}{mn} \sum_{k=1}^{mn} \lambda_k}{(\prod_{k=1}^{mn} \lambda_k )^{1/mn}} = \frac{\mathrm{AM}(\lambda_1, \dots, \lambda_{mn})}{\mathrm{GM}(\lambda_1, \dots, \lambda_{mn})}
\end{equation*}
The AM-GM inequality gives $\Phi(H,A\otimes B)\geq 1$ with equality if and only if $H \propto A \otimes B$.
$\Phi$ is also related to the \emph{coding gain} in transform coding, see \citet[Chapter 8]{gersho2012vector}.

Second, up to scaling $A \otimes B$ by some $s > 0$, it is essentially equal to the exponential of the KL divergence between Gaussians with covariances $H$ and $A \otimes B$, respectively:
\begin{equation}
\label{eq:klrelation}
\log \Phi(H, A \otimes B) = \frac{2}{m n} \min_{s > 0} \KL(\cN(0,H) \,\Vert\, \cN(0,s \cdot A \otimes B))
\end{equation}

\section{The FlipFlop Hessian Approximation}
\label{sec:flipflophessian}

In order to minimize the distortion in \cref{thm:twosidedgptq2}, we have to find a Hessian approximation $A \otimes B$ of $H$ that minimizes $\Phi(H, A \otimes B)$.
In general, this minimization problem has no closed-form solution.
But it turns out that there is an efficient algorithm that converges to the solution.
We work in the setting where $H$ is an expectation over outer products of Kronecker-factored vectors:
\begin{equation*}
H = \E_{x,g} [(x \otimes g)(x \otimes g)^T] = \E_{x,g} [xx^T \otimes gg^T]
\end{equation*}
Here $x \in \R^n$ and $g \in \R^m$, and we will drop the subscript $x,g$ from the expectation for the remainder of the section.
There is a very natural setting where this is the case, namely when $H$ is the Fisher information matrix of $W$, or equivalently the Hessian of the cross-entropy loss with respect to the network's own predictive distribution.
In that case, $x$ is a calibration input of $W$ with output $y = Wx$, a network output is sampled from the model's predictive distribution, and $g$ is the gradient of the negative log-likelihood of this network output with respect to $y$ \citep[see, e.g.,][]{birnick2026bakron}.

We now wish to find $A,B$ that minimize $\Phi(H,A \otimes B)$.
Write $G := g x^T \in \R^{m \times n}$ and note that $\vecop(G) = x \otimes g$.
From trace and Kronecker identities we get:
\begin{equation*}
\tr(H (A \otimes B)^{-1})
= \tr(\E[\vecop(G) \vecop(G)^T] (A^{-1} \otimes B^{-1}))
= \E\tr(B^{-1} G A^{-1} G^T)
\end{equation*}
Thus we can rewrite the objective $\Phi$, and also simplify it by removing constants:
\begin{equation*}
\Phi(H,A \otimes B) \quad\propto\quad |A|^{1/n} |B|^{1/m} \cdot \E\tr(B^{-1} G A^{-1} G^T)
\end{equation*}
Now we fix $A$, and ask for the best choice of $B$ given $A$.
The objective for $B$ is proportional to
\begin{equation*}
|B|^{1/m} \cdot\tr(B^{-1} S) \quad\text{with}\quad S := \E[G A^{-1} G^T] .
\end{equation*}
This is minimized when $B$ is proportional to $S$.
Indeed, the AM-GM inequality applied to the eigenvalues of $B^{-1} S$ says $\tr(B^{-1} S)/m \geq (|S|/|B|)^{1/m}$ and equality holds when the eigenvalues are all the same, leading to $B \propto S$.
Similarly, when $B$ is fixed, then $A$ minimizes $\Phi$ precisely when $A \propto \E[G^T B^{-1} G]$.
This naturally leads to the following alternating (fixed-point) iteration:
\begin{align*}
A_{k+1} &\propto \E[G^T B_k^{-1} G] = \E[g^T B_k^{-1} g \cdot xx^T] \\
B_{k+1} &\propto \E[G A_{k+1}^{-1} G^T] = \E[x^T A_{k+1}^{-1} x \cdot gg^T]
\end{align*}
In maximum likelihood estimation of matrix normal distributions, this is known as the ``flip-flop algorithm'' \citep{dutilleul1999mle}.
If $H$ is positive definite, it converges to a unique (up to scaling) fixed point \citep[Section 2]{hoff2023core}.
The connection to Gaussian maximum likelihood estimation is given by \cref{eq:klrelation}.
Hence we call the resulting approximation $A \otimes B$ of $H$ the \emph{FlipFlop} Hessian.

Interestingly, the algorithm resembles a power iteration that was previously explored for two-sided GPTQ.
Namely, if $A_{k+1}^{-1}$, $B_k^{-1}$ above are replaced by $A_{k+1}$, $B_k$, so that the weighting is in some sense reciprocal, then the algorithm becomes a power iteration that minimizes $\lVert H - A \otimes B \rVert_F$, see \citet{tseng2026modelpreserving,birnick2026bakron}.
We call the resulting approximation the \emph{Frobenius} Hessian.
It can be interpreted as the best ``additive'' Kronecker approximation to $H$, while the FlipFlop Hessian is, in a certain sense, the best ``multiplicative'' approximation.
The FlipFlop Hessian outperforms the Frobenius Hessian in practice, as suggested by the theory, see \cref{app:experiments}.

\section{Experiment Findings}
\label{sec:experiments}

We validate our theory empirically by quantizing transformer language models with the Hessian approximations from \cref{tab:hessian_comparison}.
For the concrete setup and result tables, see \cref{app:experiments}.
We find that generally the two-sided (``WaterKron'') algorithm outperform the purely input-sided (``WaterSIC'') algorithm, \emph{if the Hessian approximation is chosen carefully} like argued in the paper.
The Frobenius variant, as used for example by YAQA \citep{tseng2026modelpreserving}, generally gets \emph{worse} with more iterations (measured up to 3 iterations).
For the larger 8B models it even gets clearly worse than the purely input-sided Hessian.
This supports the argument against minimizing $\lVert H - A \otimes B \rVert_F$.
Our new FlipFlop variant generally \emph{improves} with the second iteration, and is consistently the best of all the variants, supporting the argument that $\Phi(H, A \otimes B)$ is a more appropriate metric to minimize.

\begin{table}[H]
  \caption{Different Hessian approximations $A \otimes B$ of $H = \E[xx^T \otimes gg^T]$.
  Frobenius and FlipFlop approximations are computed using the alternating fixed-point iteration, starting from the identity matrix.}
  \label{tab:hessian_comparison}
  \begin{center}
    \begin{tabular}{lccp{0.3\linewidth}}
      \hline
      Name & Input factor $A$ & Output factor $B$ & Notes \\
      \hline
      Input & $\E[xx^T]$ & $I_m$ & Recovers WaterSIC. \\
      Marginal & $\E[xx^T]$ & $\E[gg^T]$ & Treats $x,g$ as independent.
      \\
      Frobenius & $\E[(g^T B g)\,xx^T]$ & $\E[(x^T A x)\,gg^T]$ & Targets \mbox{$\lVert H-A\otimes B\rVert_F$}. \\
      FlipFlop & $\E[(g^T B^{-1} g)\,xx^T]$ & $\E[(x^T A^{-1} x)\,gg^T]$ & Targets $\Phi(H,A\otimes B)$. \\
      \hline
    \end{tabular}
  \end{center}
\end{table}

\begin{figure}[H]
  \centering
  \includegraphics[width=\linewidth]{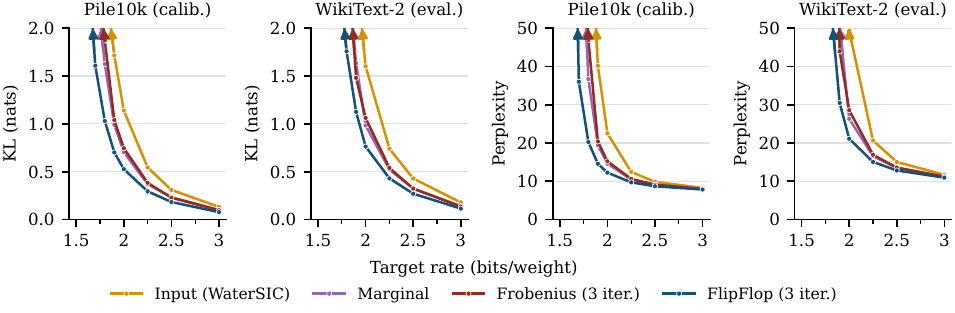}
  \caption{Llama-3.2-1B KL/PPL across rates for various Hessian approximations. (See \cref{fig:llama_1b_rates_all}.)}
  \label{fig:llama_1b_rates}
\end{figure}

\section*{Acknowledgments}

We gratefully acknowledge partial support by the National Science Foundation, via the DMS-2410717 grant.

\bibliography{waterkron}
\bibliographystyle{plainnat}

\appendix
\crefalias{section}{appendix}
\crefalias{subsection}{appendix}
\crefalias{subsubsection}{appendix}

\section{Limitations}

The main limitation for the practical use of WaterKron is carried over from Huffman-GPTQ and WaterSIC, namely the use of entropy coding.
Decoders for entropy coding schemes may be hard to implement efficiently as GPU kernels.
Practical workarounds include the use of compression schemes that are natively supported on the accelerator, or splitting $Z$ into a dense matrix of a fixed datatype and a sparse matrix of outliers.

Additionally, the theory works (naturally) in the high-rate regime, while practical neural network quantization and the experiments are in the low-rate regime.

\section{Experiment Results}
\label{app:experiments}

\textbf{Scope.}
The goal of the experiments is to \emph{isolate the WaterKron and FlipFlop Hessian contribution} by comparing it to the basic WaterSIC algorithm (i.e., the basic GPTQ algorithm with the waterfilling scaling factor choice) in a simple and clean setting without any other algorithmic refinements.
As noted already in \cref{sec:introduction}, several orthogonal improvements of GPTQ have been developed, which would likely improve the empirical results here.
These techniques include drift correction \citep[e.g.,][]{zhang2026qronos}, the use of improved base quantizers \citep[e.g.,][]{savkin2025nestquant}, finetuning the scaling factors, and combination with random/learned rotations; also see the list of techniques applied in \citet{lifar2026watersic}.
We intentionally do \emph{not} apply any of those techniques, but isolate the role of the Hessian approximation by comparing in a basic setting.

\textbf{Runs.}
We quantize and evaluate Llama \citep{grattafiori2024llama} and Qwen \citep{yang2025qwen3} models with up to 8B parameters at a target rate of 2 bits per weight, see \cref{tab:llama_1b_results,tab:llama_3b_results,tab:llama_8b_results,tab:qwen_1_7b_results,tab:qwen_4b_results,tab:qwen_8b_results}.
Additionally, we quantize Llama-3.2-1B at various rates, see \cref{fig:llama_1b_rates_all}.

\textbf{Setup.}
We only quantize linear modules from the transformer blocks, not the embedding or head matrix.
Calibration uses 256 sequences of 2048 tokens from Pile10k, a subset of The Pile \citep{pile}.
Inputs $x$ are always taken from the unquantized model, in contrast to using an $\hat{x}$ that went through the partially quantized model, as is usually done with GPTQ.
The gradients are computed in an empirical Fisher fashion, meaning that labels are not sampled from the model's predictive distribution but taken as the next token from the calibration dataset.
Hessians are damped with $0.1$ times the mean diagonal for both factors, including between iterations of the flip-flop algorithm.
The Hessian approximations that we compare are listed in \cref{tab:hessian_comparison}.
For Frobenius and FlipFlop approximations, which both use the alternating fixed-point iteration, we additionally compare different numbers of iterations.
$A$ is initialized as $I_n$, and then one iteration corresponds to updating both $B$ and $A$ once.
The number of iterations is appended as a suffix to the name, e.g., ``FlipFlop-1'', ``FlipFlop-2'', etc.
We evaluate $\KL(p_{\mathrm{unquantized}}\|p_{\mathrm{quantized}})$ (KL) and perplexity (PPL) on both the calibration dataset (Pile10k) and the WikiText-2 test split \citep{merity2016pointer}, averaged over tokens.

\begin{figure}[h]
  \centering
  \includegraphics[width=\linewidth]{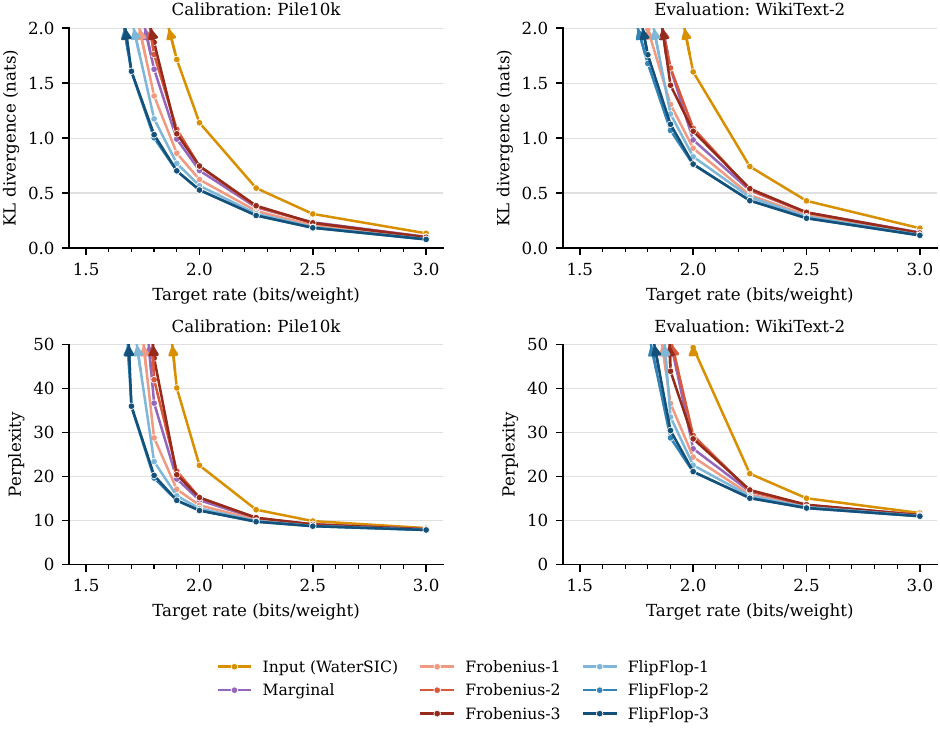}
  \caption{More detailed version of \cref{fig:llama_1b_rates}. Llama-3.2-1B quantized with all eight Hessian variants across various rates. ``FlipFlop'' outperforms all other Hessian approximations.}
  \label{fig:llama_1b_rates_all}
\end{figure}

\begin{table}[H]
  \caption{Llama-3.2-1B at a target rate of 2 bits per weight.}
  \label{tab:llama_1b_results}
  \begin{center}
    \begin{tabular}{lrrrr}
    \hline
    & \multicolumn{2}{c}{Pile10k} & \multicolumn{2}{c}{WikiText-2} \\
    \cline{2-3} \cline{4-5}
    Hessian & KL $\downarrow$ & PPL $\downarrow$ & KL $\downarrow$ & PPL $\downarrow$ \\
    \hline
    Unquantized & 0.00 & 7.3 & 0.00 & 9.8 \\
    \hline
    Input (WaterSIC) & 1.14 & 22.5 & 1.60 & 49.3 \\
    Marginal & 0.70 & 14.6 & 0.98 & 26.3 \\
    Frobenius-1 & 0.62 & 13.5 & 0.91 & 24.4 \\
    Frobenius-2 & 0.74 & 15.2 & 1.09 & 29.4 \\
    Frobenius-3 & 0.75 & 15.2 & 1.06 & 28.6 \\
    FlipFlop-1 & 0.57 & 12.7 & 0.83 & 22.6 \\
    FlipFlop-2 & 0.53 & 12.3 & \textbf{0.76} & \textbf{21.0} \\
    FlipFlop-3 & \textbf{0.53} & \textbf{12.3} & 0.76 & 21.1 \\
    \hline
    \end{tabular}
  \end{center}
\end{table}

\begin{table}[H]
  \caption{Llama-3.2-3B at a target rate of 2 bits per weight.}
  \label{tab:llama_3b_results}
  \begin{center}
    \begin{tabular}{lrrrr}
    \hline
    & \multicolumn{2}{c}{Pile10k} & \multicolumn{2}{c}{WikiText-2} \\
    \cline{2-3} \cline{4-5}
    Hessian & KL $\downarrow$ & PPL $\downarrow$ & KL $\downarrow$ & PPL $\downarrow$ \\
    \hline
    Unquantized & 0.00 & 6.2 & 0.00 & 7.8 \\
    \hline
    Input (WaterSIC) & 0.46 & 9.5 & 0.72 & 15.8 \\
    Marginal & 0.41 & 9.2 & 0.71 & 15.6 \\
    Frobenius-1 & 0.39 & 9.0 & 0.68 & 15.2 \\
    Frobenius-2 & 0.46 & 9.6 & 0.82 & 17.5 \\
    Frobenius-3 & 0.45 & 9.5 & 0.80 & 17.2 \\
    FlipFlop-1 & 0.35 & 8.6 & 0.62 & 14.4 \\
    FlipFlop-2 & \textbf{0.33} & \textbf{8.5} & \textbf{0.57} & \textbf{13.7} \\
    FlipFlop-3 & 0.33 & 8.5 & 0.58 & 13.8 \\
    \hline
    \end{tabular}
  \end{center}
\end{table}

\begin{table}[H]
  \caption{Llama-3.1-8B at a target rate of 2 bits per weight.}
  \label{tab:llama_8b_results}
  \begin{center}
    \begin{tabular}{lrrrr}
    \hline
    & \multicolumn{2}{c}{Pile10k} & \multicolumn{2}{c}{WikiText-2} \\
    \cline{2-3} \cline{4-5}
    Hessian & KL $\downarrow$ & PPL $\downarrow$ & KL $\downarrow$ & PPL $\downarrow$ \\
    \hline
    Unquantized & 0.00 & 5.3 & 0.00 & 6.2 \\
    \hline
    Input (WaterSIC) & 0.41 & 7.9 & 0.62 & 11.6 \\
    Marginal & 0.40 & 7.9 & 0.64 & 11.9 \\
    Frobenius-1 & 0.42 & 8.0 & 0.67 & 12.2 \\
    Frobenius-2 & 0.47 & 8.4 & 0.69 & 12.5 \\
    Frobenius-3 & 0.47 & 8.4 & 0.68 & 12.3 \\
    FlipFlop-1 & 0.35 & 7.5 & 0.60 & 11.4 \\
    FlipFlop-2 & 0.34 & 7.3 & 0.55 & 10.9 \\
    FlipFlop-3 & \textbf{0.34} & \textbf{7.3} & \textbf{0.54} & \textbf{10.8} \\
    \hline
    \end{tabular}
  \end{center}
\end{table}

\begin{table}[H]
  \caption{Qwen3-1.7B at a target rate of 2 bits per weight.}
  \label{tab:qwen_1_7b_results}
  \begin{center}
    \begin{tabular}{lrrrr}
    \hline
    & \multicolumn{2}{c}{Pile10k} & \multicolumn{2}{c}{WikiText-2} \\
    \cline{2-3} \cline{4-5}
    Hessian & KL $\downarrow$ & PPL $\downarrow$ & KL $\downarrow$ & PPL $\downarrow$ \\
    \hline
    Unquantized & 0.00 & 9.1 & 0.00 & 16.7 \\
    \hline
    Input (WaterSIC) & 0.78 & 17.2 & 1.11 & 37.2 \\
    Marginal & 0.74 & 16.1 & 1.09 & 36.1 \\
    Frobenius-1 & 0.64 & 14.2 & 0.90 & 28.8 \\
    Frobenius-2 & 0.76 & 15.7 & 1.07 & 32.3 \\
    Frobenius-3 & 0.77 & 15.6 & 1.08 & 32.5 \\
    FlipFlop-1 & 0.49 & 12.5 & 0.76 & 25.6 \\
    FlipFlop-2 & 0.46 & 12.3 & \textbf{0.68} & 24.0 \\
    FlipFlop-3 & \textbf{0.46} & \textbf{12.3} & 0.69 & \textbf{23.8} \\
    \hline
    \end{tabular}
  \end{center}
\end{table}

\begin{table}[H]
  \caption{Qwen3-4B at a target rate of 2 bits per weight.}
  \label{tab:qwen_4b_results}
  \begin{center}
    \begin{tabular}{lrrrr}
    \hline
    & \multicolumn{2}{c}{Pile10k} & \multicolumn{2}{c}{WikiText-2} \\
    \cline{2-3} \cline{4-5}
    Hessian & KL $\downarrow$ & PPL $\downarrow$ & KL $\downarrow$ & PPL $\downarrow$ \\
    \hline
    Unquantized & 0.00 & 7.8 & 0.00 & 13.7 \\
    \hline
    Input (WaterSIC) & 0.39 & 10.1 & 0.61 & 19.7 \\
    Marginal & 0.36 & 9.9 & 0.58 & 19.2 \\
    Frobenius-1 & 0.42 & 10.2 & 0.61 & 18.9 \\
    Frobenius-2 & 0.47 & 10.6 & 0.66 & 19.9 \\
    Frobenius-3 & 0.47 & 10.8 & 0.66 & 19.9 \\
    FlipFlop-1 & 0.34 & 9.7 & 0.54 & 18.8 \\
    FlipFlop-2 & 0.33 & 9.6 & 0.53 & 18.6 \\
    FlipFlop-3 & \textbf{0.33} & \textbf{9.6} & \textbf{0.52} & \textbf{18.3} \\
    \hline
    \end{tabular}
  \end{center}
\end{table}

\begin{table}[H]
  \caption{Qwen3-8B at a target rate of 2 bits per weight.}
  \label{tab:qwen_8b_results}
  \begin{center}
    \begin{tabular}{lrrrr}
    \hline
    & \multicolumn{2}{c}{Pile10k} & \multicolumn{2}{c}{WikiText-2} \\
    \cline{2-3} \cline{4-5}
    Hessian & KL $\downarrow$ & PPL $\downarrow$ & KL $\downarrow$ & PPL $\downarrow$ \\
    \hline
    Unquantized & 0.00 & 6.5 & 0.00 & 9.7 \\
    \hline
    Input (WaterSIC) & 0.28 & 7.9 & 0.45 & 13.0 \\
    Marginal & 0.29 & 8.0 & 0.47 & 13.1 \\
    Frobenius-1 & 0.31 & 8.0 & 0.49 & 13.6 \\
    Frobenius-2 & 0.34 & 8.2 & 0.50 & 13.6 \\
    Frobenius-3 & 0.34 & 8.2 & 0.49 & 13.4 \\
    FlipFlop-1 & 0.26 & 7.8 & 0.45 & 13.0 \\
    FlipFlop-2 & 0.25 & \textbf{7.7} & \textbf{0.42} & \textbf{12.7} \\
    FlipFlop-3 & \textbf{0.25} & 7.7 & 0.42 & 12.7 \\
    \hline
    \end{tabular}
  \end{center}
\end{table}

\section{Technical Details}
\label{app:technicalities}

For the theorems in this paper, there are two technicalities that we deliberately left without comment in order to focus on the core statements.
In this section, we clarify these technicalities.

\subsection{Encoding of the Scaling Factors \texorpdfstring{$\alpha$, $\beta$}{alpha, beta}}
\label{sec:alphabetacoding}

\cref{thm:gptq}, \cref{thm:twosidedgptq}, and \cref{thm:twosidedgptq2} only consider the encoding of $Z \in \Z^{m \times n}$ and ignore the encoding of the scaling factors $\alpha \in \R^n$ and $\beta \in \R^m$.
This is to keep the theorem statements clean.
In practice, the memory cost of the scaling factors is negligible.
For the theoretical coding argument, we can give the decoder direct access to $\alpha, \beta$, so there is no need to encode them.
In fact, the decoder can even be given full access to $A$, $B$ (which in particular determine $\alpha, \beta$), as the lower bounds of \cref{thm:itlimit} and \cref{thm:twosideditlimit} still hold true even with this more powerful decoder.

Concretely, the decoder $g$ would then be a function $g\colon \{0,1\}^* \times \R^{n \times n} \times \R^{m \times m} \to \R^{m \times n}$, taking a string and $A$, $B$, and producing $V$.
For \cref{alg:twosidedgptq}, the string would be the entropy-coded $Z$, and $\alpha,\beta$ as defined in \cref{thm:twosidedgptq} can be found from $A,B$, so the decoder can find $V = \beta \alpha^T \odot Z$.
Now, with this more powerful decoder, one could worry that the lower bound of \cref{thm:twosideditlimit} would not hold anymore, as a decoder with access to $A,B$ perhaps allows for a much more efficient scheme; but this is not the case.
Indeed, the standard proof of \cref{thm:twosideditlimit} already assumes decoder access to $A$ and $B$, see \citet{ordentlich2026highratematmul2, lifar2026watersic}.

\subsection{Limit \texorpdfstring{$\gamma \to 0$ and $mn \to \infty$}{gamma -> 0 and mn -> infinity}}
\cref{thm:gptq}, \cref{thm:twosidedgptq}, and \cref{thm:twosidedgptq2} all consider the limit $\gamma \to 0$ and $mn \to \infty$.
The limit $\gamma \to 0$ makes sure that we are in the high-rate regime where the transformed quantization errors become uniformly distributed.
The limit $m n \to \infty$, i.e., growing matrix size, is to ensure the efficiency of the coding scheme, so that it can approach the entropy of the distribution of $Z$.
Note that due to the second limit, these theorems are technically not about a fixed pair $A,B$, but about all of them at the same time.
In this case, one needs to be careful about how the joint limit $\gamma \to 0$ and $mn \to \infty$ is taken.
We understand it in an iterated sense: First take $\gamma \to 0$ for fixed $A,B$, then take $mn \to \infty$.
Then for each concrete instance of $A,B$ it is ensured that we are in the high-rate limit where the transformed quantization errors are uniformly distributed.

\end{document}